\documentclass[conference]{IEEEtran}
\IEEEoverridecommandlockouts
\usepackage{cite}
\usepackage{amsmath,amssymb,amsfonts}

\usepackage{cite}
\usepackage{amsmath,amssymb,amsfonts,amsthm}
\usepackage{algorithmic}
\usepackage{graphicx}
\usepackage{textcomp}
\usepackage{xcolor}
\usepackage{url}
\usepackage{algorithm}
\usepackage{setspace}
\usepackage{multirow}
\usepackage{booktabs}
\usepackage{subfigure}
\usepackage{stfloats}

\newtheorem{theorem}{Theorem}

\usepackage{threeparttable} 
\usepackage{algorithmic}
\usepackage{graphicx}
\usepackage{textcomp}
\usepackage{xcolor}
\def\BibTeX{{\rm B\kern-.05em{\sc i\kern-.025em b}\kern-.08em
    T\kern-.1667em\lower.7ex\hbox{E}\kern-.125emX}}

\begin{document}

\title{Efficient Model-Based Collaborative Filtering with Fast Adaptive PCA\\
\thanks{Identify applicable funding agency here. If none, delete this.}
}

\author{\IEEEauthorblockN{Xiangyun Ding\IEEEauthorrefmark{1}, Wenjian Yu\IEEEauthorrefmark{1}, Yuyang Xie\IEEEauthorrefmark{1}, Shenghua Liu\IEEEauthorrefmark{2}}
\IEEEauthorblockA{\IEEEauthorrefmark{1}Dept. Computer Science \& Tech., BNRist, Tsinghua University, Beijing, China.  \\
\IEEEauthorrefmark{2}Institute of Computing Technology, Chinese Academy of Science, Beijing, China. \\
ding-xy16@mails.tsinghua.edu.cn, yu-wj@tsinghua.edu.cn, xyy18@mails.tsinghua.edu.cn, liushenghua@ict.ac.cn
}
}


\maketitle

\begin{abstract}
A model-based collaborative filtering (CF) approach utilizing fast adaptive randomized singular value decomposition (SVD) is proposed for the matrix completion problem in recommender system. Firstly, a fast adaptive PCA framework is presented which combines the fixed-precision randomized matrix factorization algorithm \cite{yu2018efficient} and accelerating skills for handling large sparse data.
Then, a novel termination mechanism for the adaptive PCA is proposed to automatically determine a number of latent factors for achieving the near optimal prediction accuracy during the subsequent model-based CF. The resulted CF approach has good accuracy while inheriting high runtime efficiency. 
Experiments on real data show that, the proposed adaptive PCA is up to 2.7X and 6.7X faster than the original fixed-precision SVD approach \cite{yu2018efficient} and \texttt{svds} in Matlab repsectively, while preserving accuracy. The proposed model-based CF approach is able to efficiently process the MovieLens data with 20M ratings and exhibits more than 10X speedup over the regularized matrix factorization based approach \cite{koren2009matrix} and the fast singular value thresholding approach \cite{feng2018faster} with comparable or better accuracy. It also owns the advantage of parameter free. Compared with the deep-learning-based CF approach, the proposed approach is much more computationally efficient, with just marginal performance loss.
\end{abstract}


\section{Introduction}
Collaborative filtering (CF) is a key technique used by recommender systems. It makes automatic predictions (filtering) about the interests of a user by collecting perferences or taste information from many users (collaborating). There are different types of CF, such as model-based approaches \cite{koren2009matrix,sarwar2000application,libmf} and deep-learning-based approaches \cite{he2017neural,costco}. The recently developed deep-learning-based CF generalizes traditional matrix factorization models via a non-linear neural architecture and therefore often obtains better results. However, there are a number of issues in these work applying deep learning or neural methods to the recommendation problem \cite{dacrema2019we}. Another concern is that, although the deep-learning-based approaches bring accuracy improvement, they definitely cause substantial energy consumption and are very costly both financially and environmentally \cite{strubell2019energy}.


In this work, we consider the  model-based CF for the matrix completion problem in recommender system. It 
analyzes the relationships between users and items to identify new associations in user-item matrices, where latent factors for users and items are employed to represent the original data matrix. A typical approach includes a preprocessing step in which 
principal component analysis (PCA) is used to learn the low-dimensional representations. The unknown ratings are then predicted with basic collaborative filtering~\cite{ramlatchan2018survey,sarwar2000application}.
Our focus is to develop an automatic and efficient scheme for choosing the dimensionality parameter $k$ in the PCA producing latent factors. Notices that a smaller $k$ causes the inaccuracy of reduced model, while a larger $k$  induces large computational cost and possibly large error due to overfitting.

The proposed model-based CF approach is based on recent progress of randomized SVD research \cite{yu2018efficient}. It outperforms other model-based CF in terms of accuracy and runtime efficiency. Our experiments also reveal that compared with the deep-learning-based CF the proposed approach consumes much less computing resource with just a little inferiority in accuracy.
The major contributions of this work are as follows.
\begin{itemize}
\item[-] By extending the fixed-precision algorithm in \cite{yu2018efficient}, we present a fast adaptive PCA framework which automatically determines the dimensionality parameter $k$, and is accelerated for processing large sparse matrix.

\item[-] 
A CF approach based on SVD model is proposed, which includes a novel termination mechanism to make the adaptive PCA producing a suitable number of latent factors. The proposed approach exhibits high accuracy of rating prediction and runtime efficiency.

\item[-] 
Experiments on real data demonstrate that the adaptive PCA shows 2.7X and 6.7X speedup over the \texttt{randQB\_EI} algorithm \cite{yu2018efficient} and Matlab \texttt{svds} respectively, while preserving  accuracy. The proposed model-based CF approach exhibits better results on predicting ratings than the baselines including the original SVD-model based approach \cite{sarwar2000application} and fast singular value thresholding (SVT) algorithm \cite{feng2018faster}, with much less computational time (up to 33X speedup). Compared with the regularized matrix factorization (RMF) approach \cite{koren2009matrix}, the proposed algorithm is over 10X faster for achieving same accuracy, with the advantage of parameter free. Compared with the deep-learning-based CF, the proposed approach is much more computationally efficient with only 10\% or less drop of prediction accuracy.

\end{itemize}



\section{Preliminaries}
For simplicity, some Matlab conventions for matrix are used to describe algorithms.
\subsection{Singular Value Decomposition and PCA}
The singular value decomposition (SVD) of matrix $\mathbf{A}\in\mathbb{R}^{m\times n}$ is:
\begin{equation}
\mathbf{A}=\mathbf{U\Sigma V}^{\mathrm{T}},
\end{equation}
where $\mathbf{U}=[\mathbf{u}_1, \mathbf{u}_2, \cdots]$ and $\mathbf{V}=[\mathbf{v}_1, \mathbf{v}_2, \cdots]$ are orthogonal matrices which represent the left and right singular vectors, respectively. The diagonal matrix $\mathbf{\Sigma}$ contains the singular values $(\sigma_1, \sigma_2, \cdots)$ of $\mathbf{A}$ in descending order. Suppose that $\mathbf{U}_k$ and $\mathbf{V}_k$ are matrices with the first $k$ columns of $\mathbf{U}$ and $\mathbf{V}$ respectively, and $\mathbf{\Sigma}_k$ is the diagonal matrix containing the first $k$ singular values. The truncated SVD of $\mathbf{A}$ is:
\begin{equation}
\mathbf{A}_k = \mathbf{U}_k\mathbf{\Sigma}_k\mathbf{V}_k^{\mathrm{T}} ~,
\end{equation}
which is the best rank-$k$ approximation of $\mathbf{A}$ in either spectral or Frobenius norm \cite{eckart1936}.

The approximation properties explain the  equivalence between the truncated SVD and PCA. The major usage of PCA is dimensionality reduction. Roughly speaking, the rows in $\mathbf{U}_k$ and columns in $\mathbf{V}_k^\mathrm{T}$ are $k$-dimensional representations of the rows and columns of original $\mathbf{A}$, respectively. 


The built-in function \texttt{svds} in Matlab, which computes the truncated SVD, is based on a Krylov subspace iterative method and is especially efficient for sparse matrix. For a sparse matrix $\mathbf{A}\in\mathbb{R}^{m\times n}$, \texttt{svds} costs  $O(nnz(\mathbf{A})k)$ floating-point operations (\emph{flops}), where $nnz(\cdot)$ means the number of nonzero elements. 

\subsection{The Randomized SVD Algorithm and Fixed-Precision Matrix Factorization}
It has been demonstrated that the randomized methods have advantages for solving the least linear squares problem and low-rank matrix approximation \cite{drineas2016randnla}. The method for low-rank approximation mainly relies on the random projection to identify the subspace capturing the dominant actions of matrix $\mathbf{A}$, which can be realized by multiplying $\mathbf{A}$ with a random matrix on its right or left side to obtain the subspace's orthogonal basis matrix $\mathbf{Q}$. Then, the low-rank approximation  in form of $\mathbf{QB}$ is computed and further results in
 the approximate truncated SVD \cite{Halko2011Finding} (see Algorithm 1). Because $\mathbf{Q}$ has much fewer columns than $\mathbf{A}$, this method reduces the computational time. 
\begin{algorithm}
    \caption{Basic randomized SVD with power iteration}
    \label{alg1}
    \begin{algorithmic}[1]
      \REQUIRE $\mathbf{A}\in\mathbb{R}^{m\times n}$, rank parameter $k$, power parameter $p$
      \ENSURE $\mathbf{U}\in\mathbb{R}^{m\times k}$, $\mathbf{S}\in\mathbb{R}^{k\times k}$, $\mathbf{V}\in\mathbb{R}^{n\times k}$
      \STATE $\mathbf{\Omega} = \mathrm{randn}(n, k+s)$
	   \STATE $\mathbf{Q} = \mathrm{orth}(\mathbf{A}\mathbf{\Omega})$
      \FOR {$i=1, 2, \cdots, p$}
        \STATE $\mathbf{G} = \mathrm{orth}(\mathbf{A}^{\mathrm{T}}\mathbf{Q})$
        \STATE $\mathbf{Q} = \mathrm{orth}(\mathbf{A}\mathbf{G})$
      \ENDFOR
      \STATE $\mathbf{B}  = \mathbf{Q}^{\mathrm{T}}\mathbf{A}$
      \STATE $[\mathbf{U}, \mathbf{S}, \mathbf{V}] = \mathrm{svd}(\mathbf{B})$
      \STATE $\mathbf{U} = \mathbf{Q}\mathbf{U}$
      \STATE $\mathbf{U} = \mathbf{U}(:, 1:k), \mathbf{S} = \mathbf{S}(1:k, 1:k), \mathbf{V} = \mathbf{V}(:, 1:k)$
    \end{algorithmic}
  \end{algorithm}

In Alg. 1, $p$ can be $0, 1, 2, \dots$, $\mathbf{\Omega}$ is a Gaussian random matrix, 
and $s$  is an oversampling parameter (often set 10 or 20) for better accuracy.  ``orth($\cdot$)'' stands for the orthonormalization operation, which can be implemented with a call to a packaged QR factorization, e.g., \texttt{qr(X, 0)} in Matlab.  
After Step 2, the columns of $\mathbf{Q}$ become, approximately, a set of orthonormal basis of the dominant subspace of $range(\mathbf{A})$. Then, Step 7 realizes that  $\mathbf{A}\approx \mathbf{QB}=\mathbf{QQ}^{\mathrm{T}}\mathbf{A}$. By performing the economic SVD on the $(k+s)\times n$ matrix $\mathbf{B}$ the approximate $k$-truncated SVD of $\mathbf{A}$ is obtained. The power iteration (PI) scheme shown as Steps 3$\sim$6 is for improving the approximation accuracy \cite{Halko2011Finding}. It is based on that matrix $(\mathbf{AA}^{\mathrm{T}})^p\mathbf{A}$ has exactly the same singular vectors as $\mathbf{A}$, but its singular values decay more quickly. Therefore, performing the randomized QB procedure on $(\mathbf{AA}^{\mathrm{T}})^p\mathbf{A}$ can achieve better accuracy. The orthonormalization operation  in PI is used to alleviate the round-off error in floating-point computation.

The randomized PCA algorithm with PI has the following guarantee \cite{Halko2011Finding,musco2015}:
\begin{equation}
||\mathbf{A}-\mathbf{QQ}^\mathrm{T}\mathbf{A}||=||\mathbf{A}-\mathbf{USV}^\mathrm{T}||\le (1+\epsilon)||\mathbf{A}-\mathbf{A}_k||,
\label{relerr_qeuation}
\end{equation}
with a high probability. $\mathbf{A}_k$ is the best rank-$k$ approximation of $\mathbf{A}$. 


In \cite{yu2018efficient}, a fixed-precision randomized QB factorization was proposed, where $\mathbf{Q}$ and $\mathbf{B}$ with as small as possible size are sought to ensure $\|\mathbf{A-QB} \|_\mathrm{F} < \varepsilon\|\mathbf{A} \|_\mathrm{F}$. 
It is a variant of the procedure producing the $\mathbf{QB}$ approximation in Alg. 1, including an error indicator measured in Frobenius norm. It also enables more efficient adaptive rank determination for a large and/or sparse matrix, compared with the blocked randQB algorithm proposed in \cite{martinsson2016randomized2}.

\subsection{Collaborative Filtering Approaches}
A typical collaborative filtering (CF) approach in recommendation systems can be described with a matrix completion problem. Given a list of $m$ users $\{u_1,u_2,\cdots,u_m\}$ and $n$ items $\{e_1,e_2,\cdots,e_n\}$ the preferences of users toward the items can be represented as an incomplete $m\times n$ matrix $\mathbf{A}$, where each entry either represents a certain rating or is unknown. The ratings in $\mathbf{A}$ are explicit indications, such as scores given by the users in scales from 1 to 5.

A typical model-based CF is based on SVD~\cite{sarwar2000application}. It first decomposes the rating matrix $\mathbf{A}$ into a user feature matrix  and an item feature matrix with low-dimensional features, and usually can be described as the following steps.
	\begin{itemize}
		\item Perform fill in and normlization operations on the original matrix $\mathbf{A}$ and get $\mathbf{A}_{norm}$;
		\item Perform the truncated SVD on $\mathbf{A}_{norm}$ and get the latent factor matrix for items: $\mathbf{T=\Sigma}^{\frac{1}{2}}\mathbf{V}^\mathrm{T} \in \mathbb{R}^{k\times n}$;
		\item Using the latent factors in $\mathbf{T}$ to calculate similarities among items and then predict the unknown ratings. 
	\end{itemize}

 The fill in operation estimates the missing elements in $\mathbf{A}$ with simple predictions. 
The cosine similarity is usually employed in the final prediction step.
However, the optimal choice of the dimensionality $k$ of latent vector is critical to high quality prediction. It was empirically set in previous works.

Another model-based CF is the regularized marix factorization (RMF) approach~\cite{koren2009matrix}. It maps both users and items to a joint latent factor space of dimensionality $k$, such that we estimate the original matrix $\mathbf{A\approx C^{\mathrm{T}}T}$, where $\mathbf{C}\in\mathbb{R}^{k\times m}$ and $\mathbf{T}\in\mathbb{R}^{k\times n}$. $\mathbf{C}$ consists of $m$ user-related vectors $\mathbf{c}_i$ and $\mathbf{T}$ consists of $n$ item-related vectors $\mathbf{t}_j$. Then, the rating of user $i$ to item $j$ can be expressed as: 
	\begin{equation}
	   \hat{r}_{ij}=\mathbf{c}_i^{\mathrm{T}}\mathbf{t}_j ~.
	\end{equation}
The RMF based CF approach does not involve the PCA or SVD step. 
Instead, it minimizes the regularized squared error on the set of known ratings:
	\begin{equation}
	\label{RMF}
	    \min_{\mathbf{C},\mathbf{T}}\sum_{(i,j)\in S_t}(r_{ij}-\mathbf{c}_i^{\mathrm{T}}\mathbf{t}_j)^2+\lambda(\Vert \mathbf{c}_i\Vert^2+\Vert \mathbf{t}_j\Vert^2)
	\end{equation}
The $S_t$ in (\ref{RMF}) is the set of $(i,j)$ pairs for which $r_{ij}$ is known (in the training set). The constant parameter $\lambda$ is the regularized parameter and is usually determined by cross-validation. Eq. (\ref{RMF}) can be solved with the stochatic gradient descent or alternating least squares (ALS) algorithm.

Another newly proposed approach is the deep-learning based CF\cite{he2017neural}. It can be considered as an improvement of RMF approach, where the inner product is replaced with a neural network architecture learning an arbitrary non-linear function from data.

\section{Fast Adaptive PCA Framework for Processing Large Sparse Data}

The fixed-precision factorization algorithm \texttt{randQB\_EI} in \cite{yu2018efficient} does not include the power iteration. It is straightforward to combine them together to improve the proximity of the factorziation to the optimal solution produced by truncated SVD.
This derives Algorithm 2, where a block Gram-Schmidt procedure is used to realize the orthonormalization in Step 2 of Alg. 1. So, the QB factorization can be constructed in an incremental manner. Due to $||\mathbf{A}-\mathbf{QB}||_\mathrm{F}^2= ||\mathbf{A}||_\mathrm{F}^2-  ||\mathbf{B}||_\mathrm{F}^2$, the approximation error in Frobenious norm can be evaluated easily (by Steps 2 and 13 in Alg. 2) \cite{yu2018efficient}.
\begin{algorithm}
    \caption{Fixed-precision randomized QB factorization with power iteration}
    \label{alg2}
    \begin{algorithmic}[1]
      \REQUIRE $\mathbf{A}\in\mathbb{R}^{m\times n}$, relative error tolerance $\varepsilon \in (0, 1)$, block size $b$, power parameter $p$
      \ENSURE $\mathbf{Q}\in\mathbb{R}^{m\times k}$ and  $\mathbf{B}\in\mathbb{R}^{k\times n}$, such that $\|\mathbf{A-QB} \|_\mathrm{F} < \varepsilon\|\mathbf{A} \|_\mathrm{F} $
      \STATE $\mathbf{Q} = [~ ]$, ~ $\mathbf{B} = [~ ]$
      \STATE $E=  \|\mathbf{A} \|_\mathrm{F}^2 $, ~ $tol= \varepsilon^2E $
      \FOR {$i=1, 2, 3, \cdots$, }
        \STATE $\mathbf{\Omega}_i = \mathrm{rand}(n, b)$
       \STATE $\mathbf{Q}_i = \mathrm{orth}(\mathbf{A\Omega}_i-\mathbf{Q}(\mathbf{B\Omega}_i)$
        \FOR {$j=1,2,\cdots, p$}
            \STATE $\mathbf{G}_i = \mathrm{orth}(\mathbf{A}^\mathrm{T}\mathbf{Q}_i- \mathbf{B}^\mathrm{T}(\mathbf{Q}^\mathrm{T}\mathbf{Q}_i))$
            \STATE $\mathbf{Q}_i = \mathrm{orth}(\mathbf{A}\mathbf{G}_i-\mathbf{Q}(\mathbf{B}\mathbf{G}_i))$
        \ENDFOR
        \STATE $\mathbf{Q}_i = \mathrm{orth}(\mathbf{Q}_i- \mathbf{Q}(\mathbf{Q}^\mathrm{T}\mathbf{Q}_i))$  \quad   \# re-orthogonalization
        \STATE $\mathbf{B}_i=\mathbf{Q}_i^\mathrm{T}\mathbf{A}$
        \STATE $\mathbf{Q} = [\mathbf{Q}, ~\mathbf{Q}_i]$, ~ 
 $\mathbf{B} = \left[ \begin{array}{c}
             \mathbf{B}\\
             \mathbf{B}_i
        \end{array}\right]$
        \STATE $E = E - \|\mathbf{B}_i\|^2_\mathrm{F}$  
        \STATE \textbf{if} $E < tol$ \textbf{then stop}
        
    \ENDFOR

    \end{algorithmic}
  \end{algorithm}

The incremental procedure in Alg. 2 ensures that the dimensionality parameter $k$ is determined adaptively  and efficiently. To make the algorithm more general, we can replace the Frobenious-norm error indicator with other accuracy criteria. This leads to an adaptive PCA framework for more actual scenarios. 

The computational expense of Alg. 2 can be expensive for  large sparse data in real applications. 
In \cite{feng2018fast2}, several skills are presented to accelerate the randomized PCA for sparse matrix.  
A modified power iteration scheme was also proposed to allow odd number of passes over $\mathbf{A}$ and thus more convenient performance trade-off between runtime and accuracy. However, only the fixed dimensionality parameter $k$ was addressed.

We find out that the following  skills in \cite{feng2018fast2} also suit the adaptive PCA framework. 
\begin{itemize}
\item Use the eigen-decomposition for computing the economic SVD of $\mathbf{B}$.
\item Skip one orthonormalization step in each round of power iteration.
\item Use LU factorization to replace the orthonormalization in the power iteration (except the last round of iteration).
\item Use the modified power iteration to allow odd number of passes over $\mathbf{A}$.
\end{itemize}
With them, we modify the inner-loop steps in Alg. 2 and derive  Algorithm 3, which is a fast adaptive PCA framework for processing large sparse data. Notice we only present the algorithm for the situation with $m\le n$. If $m>n$, a similar efficient algorithm can be derived.
\begin{algorithm}[h]
  \caption{Fast adaptive PCA framework for sparse data} 
  \label{rPCA}
  \begin{algorithmic}[1] 
    \REQUIRE sparse matrix $\mathbf{A}\in\mathbb{R}^{m\times n}$~$(m \le n)$, block size $b$, pass parameter $q > 2$
    \ENSURE $\mathbf{U}\in\mathbb{R}^{m\times k}$, $\mathbf{S}\in\mathbb{R}^k$, $\mathbf{V}\in\mathbb{R}^{k\times n}$ for certain accuracy criterion
      \STATE $\mathbf{Q} = [~ ]$, ~ $\mathbf{B} = [~ ]$
      \FOR {$l=1, 2, 3, \cdots$, }
      \IF {$q$ is an even number}
        \STATE $\mathbf{\Omega}=$ randn($n,b$)
        \STATE $\mathbf{Y}=\mathbf{A}\mathbf{\Omega}-\mathbf{Q}(\mathbf{B}\mathbf{\Omega})$
        \STATE $[\mathbf{Q}_l,  \sim]=$lu($\mathbf{Y}$) \quad \quad \quad \quad \quad \quad \quad \# LU factorization 
      \ELSE    
        \STATE $\mathbf{Q}_l=$randn($m,b$) \quad \quad \quad \# when $q$ is an odd number
      \ENDIF
      \FOR {$t=1,2,\cdots,\lfloor\frac{q-1}{2}\rfloor$}
        \IF {$t==\lfloor\frac{q-1}{2}\rfloor$}
          \STATE $\mathbf{R}=\mathbf{A}^\mathrm{T}\mathbf{Q}_l$ ~ \quad \quad \# remove one orthogonalization 
          \STATE $\mathbf{Q}_l= \mathrm{orth}(\mathbf{A}\mathbf{R}-\mathbf{Q}(\mathbf{B}\mathbf{R}))$ \quad \# orthogonalization 
        \ELSE
          \STATE $[\mathbf{Q}_l, \sim] =$lu($\mathbf{A}(\mathbf{A}^\mathrm{T}\mathbf{Q}_l)$) \quad \quad \quad \# LU factorization
        \ENDIF
      \ENDFOR
      \STATE $\mathbf{Q}_l = \mathrm{orth}(\mathbf{Q}_l-\mathbf{Q}(\mathbf{Q}^\mathrm{T}\mathbf{Q}_l))$
      \STATE $\mathbf{B}_i=\mathbf{Q}_l^\mathrm{T}\mathbf{A}$
      \STATE $\mathbf{Q}=\begin{bmatrix}\mathbf{Q}&\mathbf{Q}_l\end{bmatrix}, ~ \mathbf{B}=\begin{bmatrix}\mathbf{B}\\\mathbf{B}_l\end{bmatrix}$
    \IF {termination criterion is met}
    \STATE $k$ is determined and then break
    \ENDIF
    \ENDFOR
  \STATE $[\mathbf{\hat{U}},\mathbf{\hat{S}},\mathbf{\hat{V}}] = \mathrm{eigSVD}(\mathbf{B^\mathrm{T}})$ \quad \quad  \# compute economic SVD
    \STATE $ind=lb:-1:lb-k+1$
  \STATE $\mathbf{U}=\mathbf{Q}\mathbf{\hat{V}}(:, ind),  ~ \mathbf{S}=\mathbf{\hat{S}}(ind), ~ \mathbf{V}=\mathbf{\hat{U}}(:, ind)$
	\end{algorithmic}
\end{algorithm}

This framework allows general termination criterion (set at Step 21) to determine the parameter $k$ and stop the incremental generation of $\mathbf{Q}$. The pass parameter $q$ means the number of passes over $\mathbf{A}$ in the algorithm, which is different from the power parameter $p$ in Alg. 2. 
``eigSVD'' denotes the algorithm using eigen-decomposition to compute economic SVD \cite{feng2018fast2}. 
In Steps 26 and 27 the tail columns of $\mathbf{\hat{U}}, ~ \mathbf{\hat{S}}$ and $\mathbf{\hat{V}}$ are used, because the eigen-decomposition of a symmetric matrix produces
 eigenvalues in ascending order.

The equivalence between the adaptive PCA framework and the incremental iteration steps in the  
fixed-precision QB factorization (Alg. 2) can be established.
\begin{theorem}
Providing that the number of iterations in the outer loop is the same and $q=2p+2$, the $\mathbf{Q}$ and $\mathbf{B}$ obtained from the adaptive PCA framework (Alg. 3) are the same as those obtained from the fixed-precision randomized QB factorization (Alg. 2) in exact arithmetic.
\end{theorem}
\begin{proof}
When $q=2p+2$, the number of power iteration is the same for the both algorithms. There are two main differences between Alg. 2 and Alg. 3. The first one is at Step 6 of Alg. 3, which is due to Lemma 3 in \cite{feng2018fast2} saying that the ``orth($\cdot$)'' operation in the power iteration can be replaced by LU factorization. It also explains the usage of ``lu($\cdot$)'' in Steps 6 and 15. 
The other is that two steps in the power iteration are combined in a single one at Step 15 of Alg. 3. This is analyzed in \cite{rsvdpack} and \cite{feng2018fast2} showing that the orthonormalization or LU factorization can be performed after every other matrix multiplication while preserving accuracy. So, the two algorithms produce same results.
\end{proof}

Actually, the steps in both algorithms ensure that $\mathbf{Q}$ is an orthonormal matrix and it includes a set of orthonormal basis of subspace $range((\mathbf{AA}^\mathrm{T})^p\mathbf{A\Omega})$. According to Lemma 2 in \cite{feng2018fast2}, the resulted $\mathbf{Q}$ and $\mathbf{B}$ are also the same as those obtained from the basic randomized SVD (Alg. 1).



\section{Model-Based Collaborative Filtering with Automatic Determination of Latent Factors}

In this section, we propose a model-based CF approach which leverages the 
fast adaptive PCA framework in last section and automatically determines a number of latent factors for potentially best prediction accuracy.

We first review the basic CF approach for predicting ratings, which is presented as Algorithm 4. Here we assume the latent factors for items have been obtained. 
\begin{algorithm}
  \caption{The basic CF for predicting $\mathbf{A}(i,j)$} 
  \label{NNT4R}
  \begin{algorithmic}[1] 
    \REQUIRE user-item matrix $\mathbf{A}\in\mathbb{R}^{m\times n}$, latent factor matrix for items $\mathbf{T}\in\mathbb{R}^{k\times n}$ , $i, ~ j$
    \ENSURE predicted rating $\hat{\mathbf{A}}(i,j)$
    \STATE $a=0$, ~ $w=0$
    \FOR {$l=1,2,\cdots,n$}
      \IF {$\mathbf{A}(i,l)$ is a known value}
        \STATE $\gamma = \frac{\mathbf{T}(:,j)^\mathrm{T}\mathbf{T}(:,l)}{\| \mathbf{T}(:,j)\|\cdot\| \mathbf{T}(:,l)\| }$ \quad \quad \quad \quad \quad \# cosine correlation 
        \STATE $a=a+\gamma \mathbf{A}(i,l)$
        \STATE $w=w+\gamma$
      \ENDIF
    \ENDFOR
    \STATE $\hat{\mathbf{A}}(i,j)=a/w$
	\end{algorithmic}
\end{algorithm}
	
To derive the CF approach with the number of latent factors automatically determined, we propose to employ the fast adaptive PCA framework and make it terminated with a suitable accuracy criterion. Firstly, we remove the fill in operation in the original SVD-based CF approach \cite{sarwar2000application}, because it causes large computation of SVD on a dense matrix and possible data distortion. Based on the efficient Alg. 3 for sparse matrix, our approach adapts to very large dataset.
Then, we use a validation set to construct the suitable accuracy criterion for the adaptive PCA. 
With the known ratings in the validation set we can keep checking the prediction error after performing CF. Once the error is small enough or reaches a local minimal, the adaptive PCA terminates and the latent factors are obtained. 
    
For evaluating the accuracy (or error) of rating prediction, there are several kinds of error measure. Without loss of generality, we consider the mean absolute error (MAE),
\begin{equation}
    \mathrm{MAE}=\frac{\sum_{(i,j)\in S}\vert\mathbf{A}(i,j)-\mathbf{\hat{A}}(i,j)\vert}{\vert S\vert} ~ ,
\end{equation}
where $S$ is a set of matrix indices for validation or testing, and $\vert S\vert$ is its cardinality. $\mathbf{A}(i,j)$ and $\mathbf{\hat{A}}(i,j)$ are the true value and predicted value of a matrix entry, respectively.

If MAE is used as error measure, the automatic approach for determining the latent factors is presented as Algorithm 5. 
There, ``$\mathrm{spdiags}$'' denotes the function for constructing a sparse diagonal matrix, and is used to compute the latent factor matrix for items $\mathbf{T}$.
In Alg. 5, $\mathbf{T}$ is successively expanded, each time with  $b$ more latent factors ($b$ is usually set 10 or 20). More latent factors mean larger computational cost in both the PCA stage and the following CF procedure. So, we monitor the prediction error on the validation set, and therefore we can terminate the adaptive PCA in an early time which ensures good accuracy as well.
Finally, the algorithm outputs $\mathbf{T}$ with a suitable number of latent factors.


\begin{algorithm}[htbp]
  \caption{Automatic determination of latent factors for the model-based CF} 
  \label{MFRS}
  \begin{algorithmic}[1] 
    \REQUIRE $\mathbf{A}\in\mathbb{R}^{m\times n}$, validation set $S_v$, block size $b$, pass parameter $q > 2$
    \ENSURE $\mathbf{T}\in\mathbb{R}^{k\times n}$
      \STATE $\mathbf{Q} = [~ ]$, ~ $\mathbf{B} = [~ ]$
      \FOR {$l=1, 2, 3, \cdots$, }
      \STATE Steps 3$\sim$20 of Alg. 3 or its variant for $m>n$
        \STATE $[\mathbf{\hat{U}}, ~ \mathbf{\hat{S}},~ \sim] = \mathrm{eigSVD}(\mathbf{B^\mathrm{T}})$
      \STATE $\mathbf{T}=\mathrm{spdiags}(\mathrm{sqrt}(\hat{\mathbf{S}}),0,bl,bl)\hat{\mathbf{U}}^\mathrm{T}$  \\ ~ \quad \quad \quad \quad \quad \quad \quad \quad \# for $\mathbf{B} \approx\mathbf{U\Sigma V}^\mathrm{T}$,  $\mathbf{T=\Sigma}^{\frac{1}{2}}\mathbf{V}^\mathrm{T} $
      \STATE $\mathrm{MAE}_l=0$
      \FOR {$(i, j)\in S_v$}
        \STATE $\tilde{\mathbf{A}}(i,j)= $ $\mathbf{A}(i,j)$ predicted with Alg. 4 and $\mathbf{T}$  
        \STATE $\mathrm{MAE}_l = \mathrm{MAE}_l + \left|\tilde{\mathbf{A}}(i,j)-\mathbf{A}(i,j)\right|/\left|S_v\right|$
      \ENDFOR
      \STATE \textbf{if} $\mathrm{MAE}_l$ is small enough or the minimal \textbf{then} break 
    \ENDFOR
    \STATE $k=bl, ~ \mathbf{T}= \mathbf{T}(k,:)$ 
	\end{algorithmic}
\end{algorithm}
The computational cost of Steps 7 $\sim$ 10 in Alg. 5 is about $k\cdot nnz(\mathbf{A})\left|S_v\right| /m$, where $k$ is the dimensionality or the row number of $\mathbf{T}$. 
In practice, we use a small portion of the known ratings as the validation set, and $\mathbf{A}$ is very sparse. So, it adds small extra computation to the adaptive PCA procedure (Alg. 3). Since the validation data and the unknown ratings are from same dataset, we expect the obtained latent factors lead to near optimal accuracy on predicting the unknown ratings. This will be validated in the experiment section.

In summary, the proposed model-based CF approach consists of the following major steps.
\begin{itemize}
\item Use Alg. 5 to generate the optimal number of latent factors $\mathbf{T}$ for the sparse user-item matrix $\mathbf{A}$.
\item Use the basic collaborative filtering approach (Alg. 4) to predict unknown ratings.
\end{itemize}

\section{Experimental Results}
The proposed algorithms for generating latent factors (Alg. 3 and 5) have been implemented in Python with Scipy package.
The collaborative filtering algorithm for predicting ratings (Alg. 4) is implemented in C, and invoked by the Python program. The adaptive PCA framework for sparse matrix is compared with \texttt{svds} in Matlab and the \texttt{randQB\_EI} algorithm combined with power iteration \cite{yu2018efficient}, with same truncated $k$. The proposed model-based CF approach with automatic determination of latent factors is compared with the original SVD model based approach \cite{sarwar2000application}, the RMF approach \cite{koren2009matrix}, LibMF \cite{libmf} and the fast SVT matrix completion algorithm \cite{feng2018faster}\footnote{We used its codes shared at https://github.com/XuFengthucs/fSVT}. The SVD model based approach and the RMF approach are implemented in Python by us, while LibMF is the codes shared by its authors \cite{libmfweb}.
We use mean absolute error (MAE) to measure the accuracy of rating prediction.

Unless otherwise stated, the experiments are carried out on a computer with Intel Xeon CPU @2.00 GHz and 128 GB RAM. The CPU times of different algorithms are compared for the fairness. The block size in Alg. 2, Alg. 3 and Alg. 5 are all set as $b=20$.

\subsection{Datasets and Settings}
Four sparse matrices are tested in our experiments. They are from the datasets: MovieLens\footnote{https://grouplens.org/datasets/movielens/} \cite{movielens}, hetrec2011\footnote{https://grouplens.org/datasets/hetrec-2011/} \cite{Cantador:RecSys2011} and BookCrossing\footnote{http://www2.informatik.uni-freiburg.de/\textasciitilde cziegler/BX/} \cite{ziegler2005improving}. 
Two matrices are from MovieLens, with about 100K and 20M ratings respectively. The matrices include the ratings of some users to items (products). The details about the test matrices are listed in Table \ref{datasets}. 

The BookCrossing matrix originally has 105,283 users, 340,556 items (books) and 1,149,780 ratings. However, there are many rows and columns with very few ratings.  In order to make the recommandation more meaningful, we removed rows and columns with less than 2 ratings from the matrix. The resulted matrix reflecting 21,795 users and 48,631 items and including 249,533 ratings is tested.

\begin{table}[ht]
	\centering
	\caption{Details of the tested matrices in our experiments.}
\small{
		\begin{tabular}{@{}c@{~}c@{~}c@{~}c@{~}c@{~}c@{}}  
			\toprule
			Test Matrix &\#Users ($m$) &\#Items ($n$) &\#Ratings& Rating Range\\
			\midrule
			MovieLens-100K&610&9,724&100,836& [0.5, 5]\\
			hetrec2011 &2,113&10,109&855,598& [0.5, 5]\\
			BookCrossing&21,795&48,631&249,533& [0.5, 10]\\
			MovieLens-20M&138,492&26,744&20,000,263 & [0.5, 5]\\
			\bottomrule
		\end{tabular}
}
	\label{datasets}
\end{table}

\subsection{Fast Adaptive PCA Framework for Sparse Data}
In this subsection, we test the proposed fast adaptive PCA framework for sparse data (i.e. Alg. 3) to show its accuracy and efficiency. To ensure sufficient accuracy, the pass parameter in Alg. 3 is set $q=10$. This is equivalent to setting $p=4$ in  Alg. 2. We run them with a given $k$ values as termination criterion
to perform PCA for the test matrices. 
The computed singular values are plotted in Fig. 1, where the accurate results 
from \texttt{svds} are also given. From the figure,
\begin{figure}[h]
  \centering
  \subfigure[MovieLens-100K] {\includegraphics[width=1.7in]{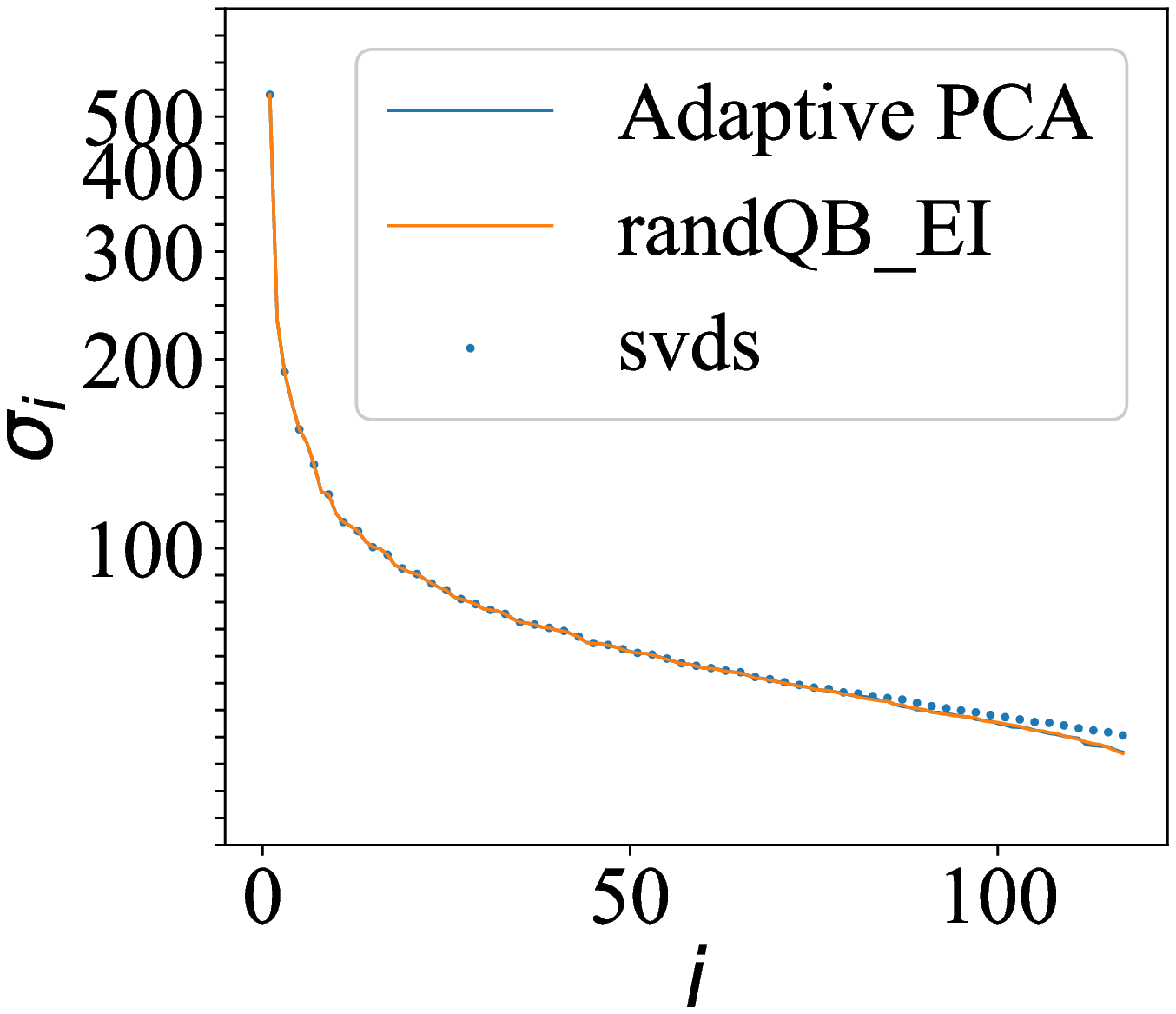}}
  \subfigure[hetrec2011] {\includegraphics[width=1.7in]{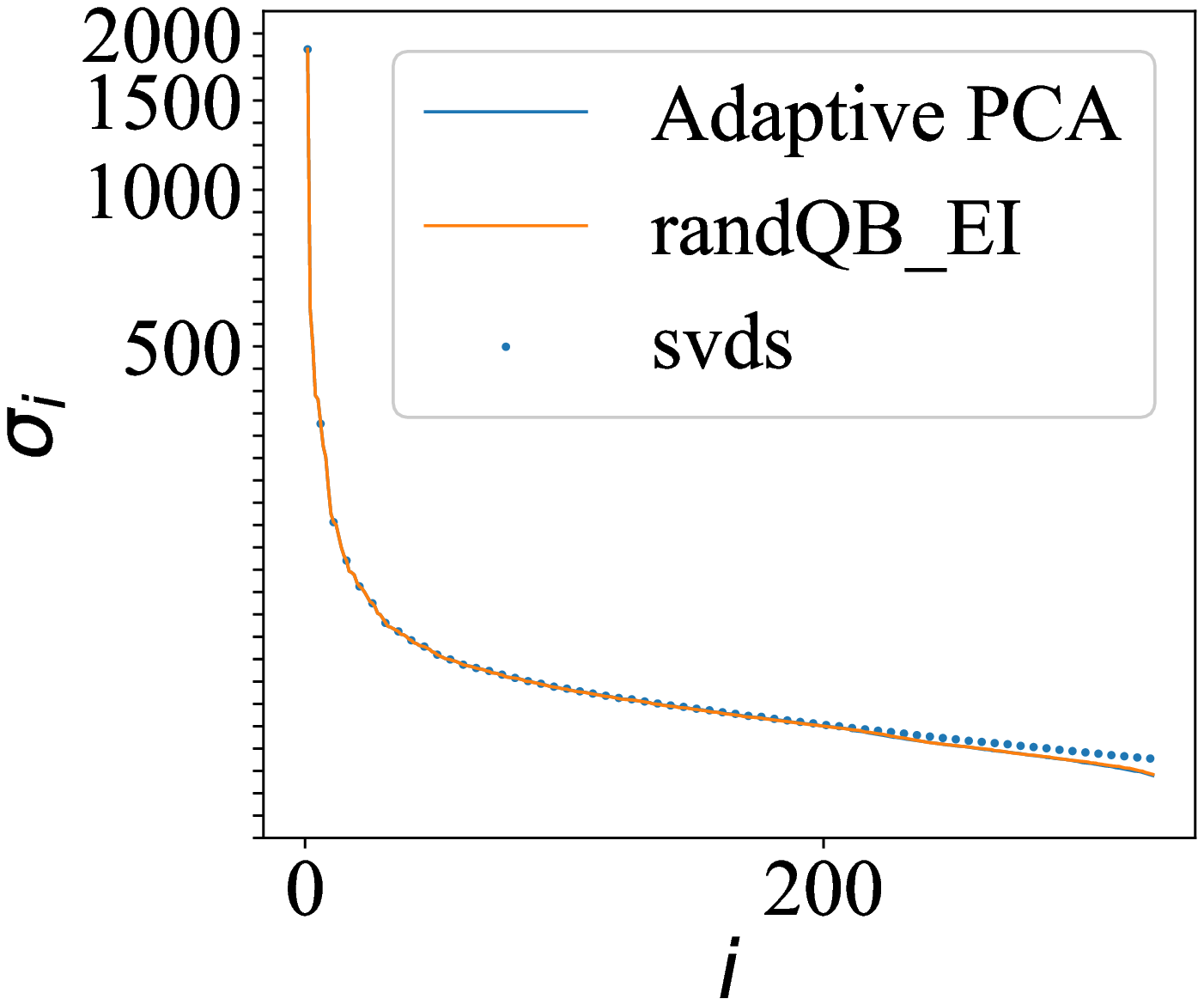}}
  \subfigure[BookCrossing] {\includegraphics[width=1.7in]{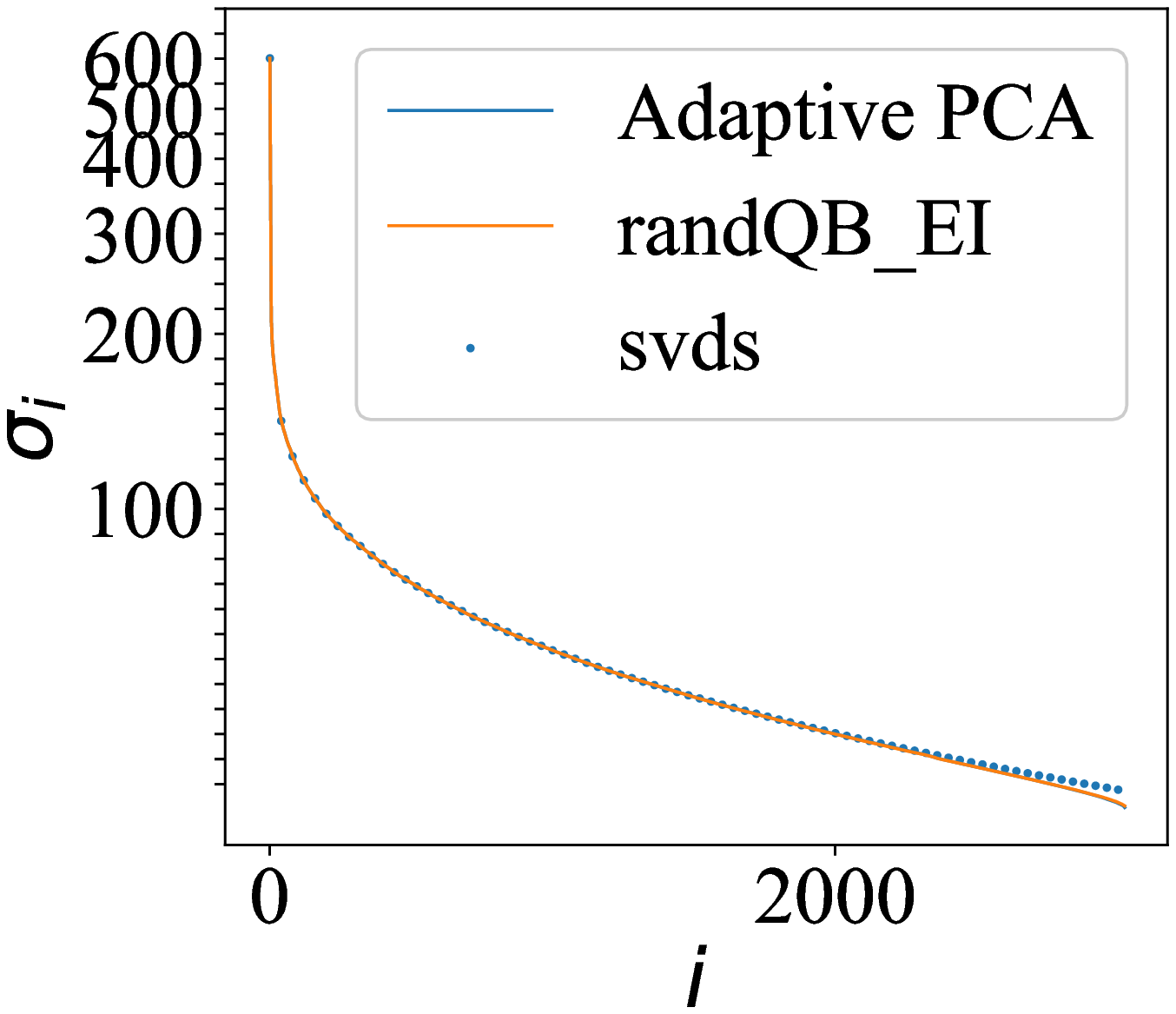}}
  \subfigure[MovieLens-20M] {\includegraphics[width=1.7in]{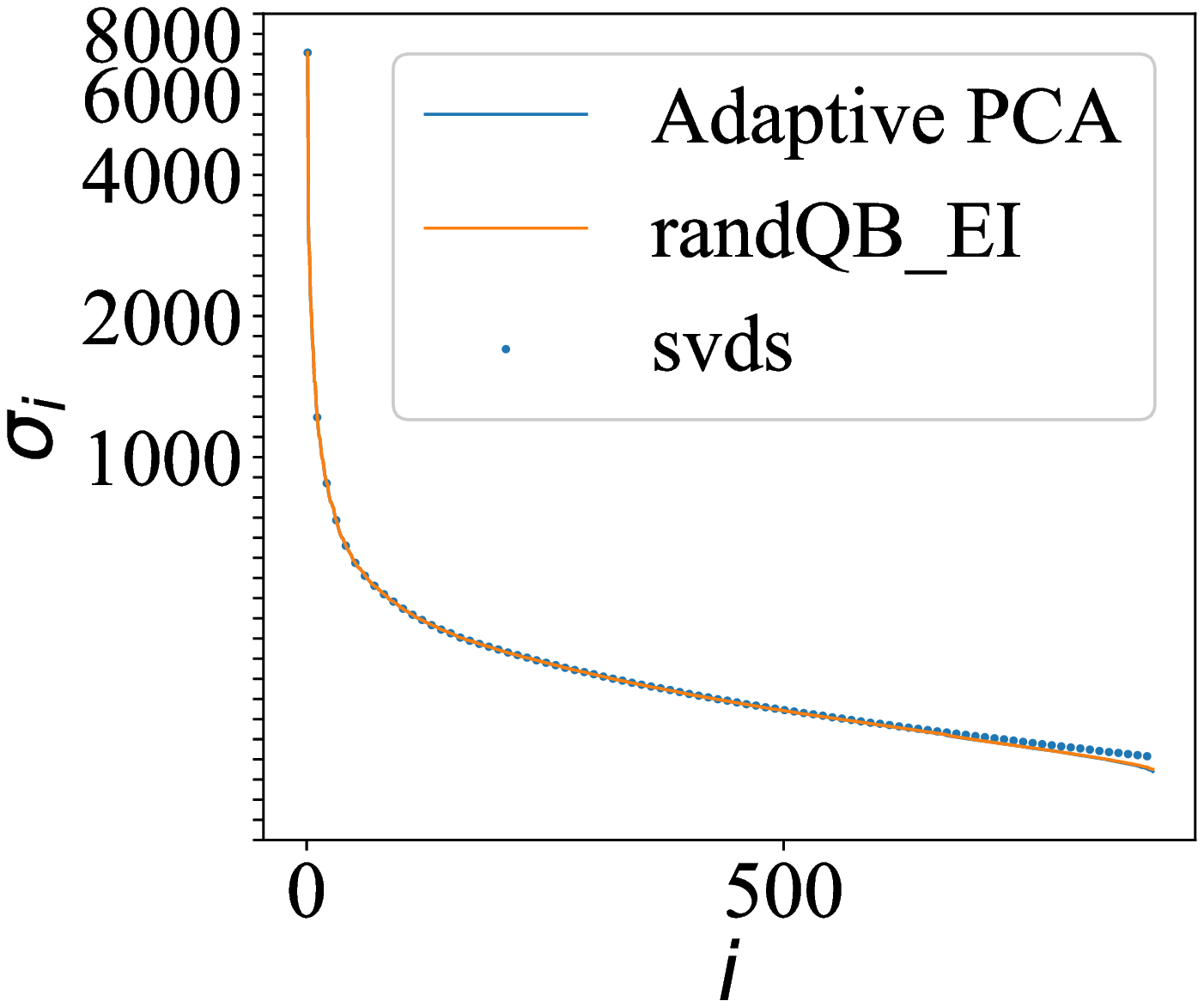}}
  \caption{The computed singular values for different matrices, showing the accuracy of proposed fast adaptive PCA framework for sparse data.}
  \label {Fig1}
\end{figure}
we see that the results of adaptive PCA framework are indistinguishable from those of \texttt{randQB\_EI}, while the both match those from \texttt{svds} very well.
Notice for two cases up to 883 and 3014 singular values are computed, and our adaptive PCA framework  still exhibits very good accuracy.


Then, we test the fast adaptive PCA with the termination criterion of approximation error.  The matrix approximation error in Frobenius norm is checked so that it can be compared with the \texttt{randQB\_EI} algorithm \cite{yu2018efficient}, i.e. Alg. 2. 
The relative error tolerance is set $\varepsilon=0.5$. 
The results are also compared with those from \texttt{svds}, as listed in Table \ref{svd}.
Because \texttt{svds} cannot adaptively determine the $k$ value, we just set the same $k$ as our approach and then record its computational time. From the table we see that, our adaptive PCA framework runs faster than the \texttt{randQB\_EI} and \texttt{svds}, with the largest speedup of \textbf{2.8X} and \textbf{6.7X} respectively. As for the determined $k$ value, the proposed algorithm mostly results to same $k$ values as \texttt{randQB\_EI} algorithm, which are just slightly larger than the optimal ranks $k^*$. Notice that $k^*$ cannot be directly outputted by \texttt{svds}, and we have to run it with a sufficiently larger $k$ and then determine $k^*$ by processing the computed leading singular values. In this meaning, the advantage of our fast adaptive PCA over  \texttt{svds} is even larger.

The memory cost of the  adaptive PCA  is the same as \texttt{randQB\_EI}, just a little bit larger than that of \texttt{svds}. Take the largest case MovieLens-20M as an example. The former costs 1.4 GB memory while \texttt{svds} consumes 1.2 GB memory.

\begin{table*}[t]
	\centering
	\caption{The results of proposed fast adaptive PCA framework (Alg. 3) with Frobenius-norm approximation error $\varepsilon=0.5$, and the comparison with \texttt{randQB\_EI} (Alg. 2) and \texttt{svds}.}
		\begin{threeparttable}
		\begin{tabular}{ccccccccccc}
			\toprule
			\multirow{2}{*}{Matrix} &
            \multicolumn{2}{c}{\texttt{randQB\_EI}(Alg. 2)} & & \multicolumn{2}{c}{\texttt{svds}}& &\multicolumn{4}{c}{Proposed Fast Adaptive PCA} \\
            \cmidrule(r){2-3}
            \cmidrule(r){5-6}
            \cmidrule(r){8-11}
            & time (s) & $k$ & & time (s) & $k^*$ & & time (s) & $k$ & Sp1 & Sp2\\
			\midrule
			MovieLens-100K&3.83&118& &3.0&117& &{\bf 1.75}&118 & 1.8 & 1.7\\
			hetrec2011&40.3&328& &71.8&325& &{\bf 29.1}&327 & 1.4 & 2.5\\
			BookCrossing&2739&3012& &6758&3004& &{\bf 1004}&3014 & {\bf 2.8} & {\bf 6.7}\\
			MovieLens-20M&1476&883& &3722&879& &{\bf 704}&883 & 2.1 & 5.3\\
			\bottomrule
		\end{tabular}
\begin{tablenotes}
    \small 
      \item[]$k^*$ is the optimal rank to truncate SVD to meet the approximation criterion. Sp1 and Sp2 are the speedup ratios to \texttt{randQB\_EI} and \texttt{svds}, respectively.
      \end{tablenotes}
	\end{threeparttable}		
	\label{svd}
\end{table*}

\subsection{The Proposed Model-Based Collaborative Filtering}
In this subsection, we validate the proposed model-based CF approach with automatically determined number of latent factors (Alg. 4 + Alg. 5). For each rating matrix, we split its  nonzero elements (known ratings) and obtain a training set, a validation set and a test set.  The training set is the input of the CF approaches. The validation set provides the measure of prediction accuracy to Alg. 5. The test set is for evaluating the accuracy of the proposed CF approach and its counterparts. In the experiment, for each matrix we randomly sample 90\% ratings to form the training set, 5\% ratings to form the validation set, and 5\% ratings for the test set. 

We first evaluate the impact of the dimensionality parameter $k$ (i.e. the number of latent factors) on the prediction accuracy. Fig. \ref{mae_k} shows the curves of MAE on validation set and test set of MovieLens-20M and the 
corresponding computational time of the proposed CF approach with successively augmented latent factors. The MAE drops quickly when $k$ is within 100, and then gradually converges to the minimal MAE. While $k$ increases further, the MAE slightly increases. From the figure we see that the minimal MAE occurs at close $k$ values for the validate set and test set. This means using the validate set in Alg. 5 can well model the trend of prediction accuracy on the test set. As 
\begin{figure}[h]     
  \centering
  \includegraphics[width=2.8in]{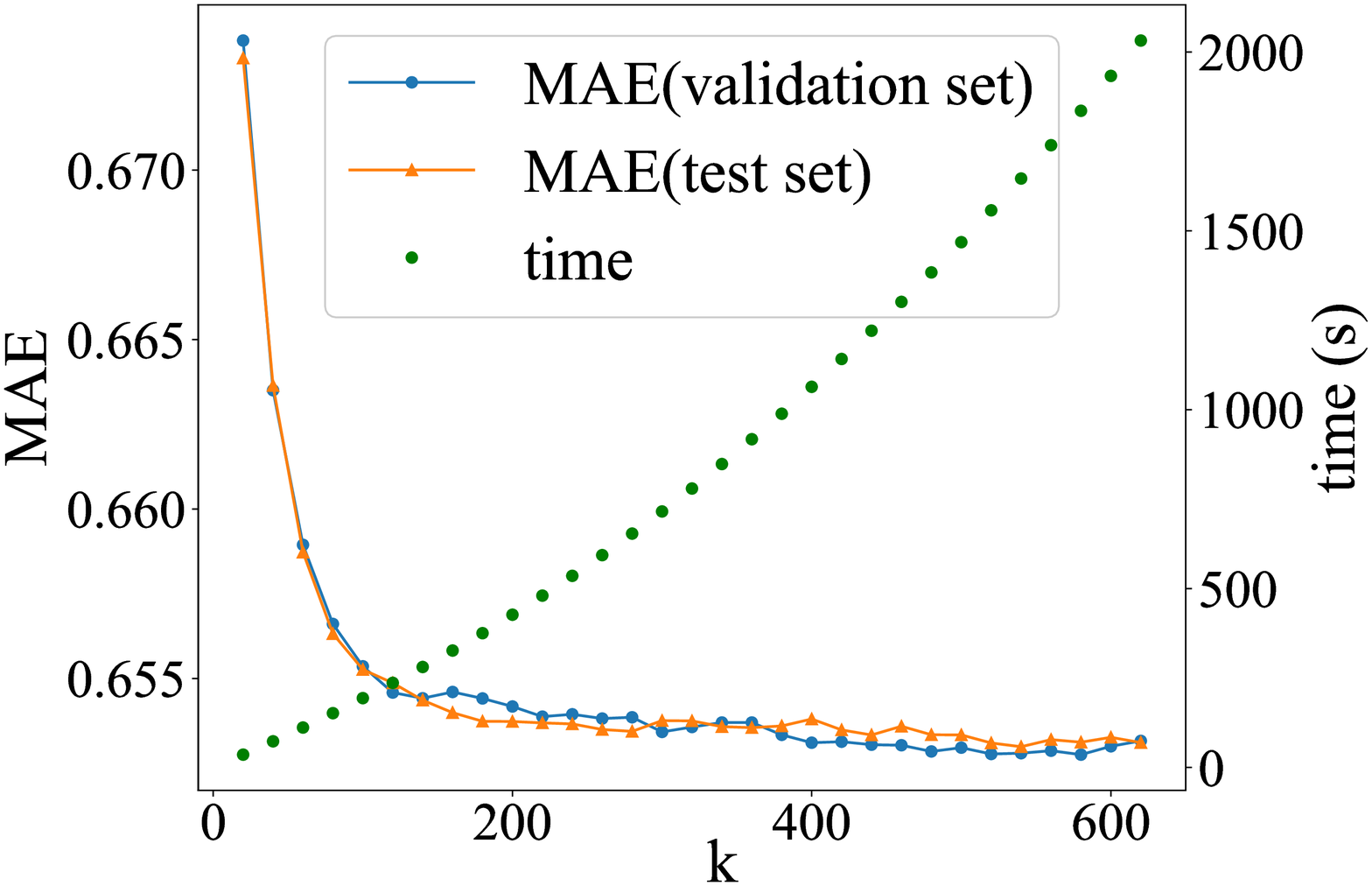}
\caption{The MAE on validation set and test set, and the corresponding computational time of the $k$-adpative rPCA based rating prediction vs. parameter $k$. The minimal MAE on validation set occurs when $k=520$, and on test set when $k=540$. The test data is from MoiveLens-20M.}
\label {mae_k}
\end{figure}
for the computational time, the curves exhibits an approximately linear dependence on $k$. It validates the efficiency of the incremental latent factor generation approach in Alg. 5.



Then, we compare the proposed approach with the SVD model based CF \cite{sarwar2000application}, the RMF approach \cite{koren2009matrix}, libMF \cite{libmf}, the fast SVT algorithm  \cite{feng2018faster} and NeuMF \cite{he2017neural}. The SVD model or RMF based approaches need setting a suitable $k$ value. In the experiment, it is set the same as that determined by the proposed approach, which corresponds a minimal prediction error on the validation set.
This ensures fair comparison and the best performance of them as well. LibMF is an improved version of RMF based approach. 
For them we have tuned the regularization parameter $\lambda$ and learning rate $\gamma$ to obtain the best accuracy of prediction. The resulted setting is $\lambda=0.05$ and  $\gamma=0.01$.
The fast SVT algorithm is a matrix completion algorithm which solves the matrix rank minimization problem and employs randomized SVD techniques to make acceleration. For the fast SVT algorithm, a parameter of tolerance is set 0.2. The NeuMF approach is a deep-learning-based CF approach. We have implemented it and obtained similar performance as that reported in \cite{he2017neural}. 

The experimental results are listed in Table \ref{recommand_result}. The MAE and CPU time are listed for all methods. 
\begin{table*}[bp]
	\centering
	\caption{The errors of rating prediction  and corresponding computational time (in unit of second) with different CF approaches.}
	\begin{threeparttable}
		\begin{tabular}{@{~}c@{~}c@{~}c@{~}c@{~}c@{~}c@{~}c@{~}c@{~}c@{~}c@{~}c@{~}c@{~}c@{~}c@{~}c@{~}c@{~}c@{~}c@{~}c@{~}c@{~}}
			\toprule
			\multirow{2}{*}{Dataset} &
            \multicolumn{3}{c}{Proposed Approach} & & \multicolumn{3}{c}{RMF Model} & &\multicolumn{2}{c}{LibMF}&&\multicolumn{2}{c}{SVD Model}& &\multicolumn{2}{c}{Fast SVT}& &\multicolumn{2}{c}{NeuMF}\\
            \cmidrule(r){2-4}
            \cmidrule(r){6-8}
            \cmidrule(r){10-11}
            \cmidrule(r){13-14}
            \cmidrule(r){16-17}
            \cmidrule(r){19-20}
            & MAE & time(s) & $k$ & & MAE & time(s) & time$^1$ & & MAE & time (s) & & MAE & time(s)&&MAE&time(s)&&MAE&time(s)\\
			\midrule
			MovieLens-100K&\textbf{0.661}&22.6&220 &&0.706&1510&NA &&0.865 &52.4&&0.754&60.6 &&0.744&459&&\underline{0.630}&NA$^2$\\
			hetrec2011&0.577&60.7&60 &&\textbf{0.577}&6542&1662 &&0.634&431&&0.702&106 &&0.621&1269&&\underline{0.559}&NA$^2$\\
			BookCrossing&1.42&107&160 &&\textbf{1.33}&4488&804 &&1.51&445&&1.58&1258 &&1.60&3552&&\underline{1.25}&NA$^2$\\
			MovieLens-20M&0.653&2114&520 &&\textbf{0.632}& NA$^2$&20154 &&0.758&10347&&NA$^2$&NA$^2$ &&0.669&6080&&\underline{0.603}&NA$^2$\\
			\bottomrule
		\end{tabular}
     \begin{tablenotes}
       \small
\item[] $^1$ The time for reaching same accuracy as the proposed approach. 
$^2$ Costs more than 40000 seconds, or causes out-of-memory error.
\end{tablenotes}
   \end{threeparttable}
	\label{recommand_result}
\end{table*}
Obvious, the NeuMF achieves the best accuracy (but the advantage is mostly less than 5\%). However, it costs multiple tens of thousands seconds for training an epoch, and usually 20$\sim$50 epoches are needed to attain the best/converged MAE. The runtime of NeuMF is much longer than the other model-based CF approaches, making it unfair to compare them together. 

Excluding the results of NeuMF, the most accurate results are indicated with bold numbers, which are from RMF model and the proposed approach.
From the results we see that the proposed approach outperforms the SVD based model, LibFM, and the fast SVT algorithm both on accuracy and computational time. The speedup ratio ranges from 2.3 to  several tens. For the largest data MovieLens-20M, the SVD based model failed due to excessive memory cost.

The RMF approach implemented by us is more than 10X slower than LibMF, but  produces more accurate result with smaller MAE. Notice that its computational time is tens to hundreds times longer than ours. To make fair comparison, the costed time for the RMF approach to reach the same accuracy as the proposed approach is also listed in the table (denoted as time$^1$). Based on it we see that the latter is \textbf{8X} to \textbf{27X} faster. For the largest MovieLens-20M case, the curve of MAE along the computational time of RMF approach is shown in Fig. 3, along with the result of the proposed approach. The figure clearly demonstrates that 
the proposed approach is about \textbf{10X} (i.e. 20154s/2114s) faster then the RMF based CF approach while producing same accuracy. And, the accuracy of the former is just a little bit worse than the converged result of the latter.
\begin{figure}[h]     
  \centering
  \includegraphics[width=2.7in]{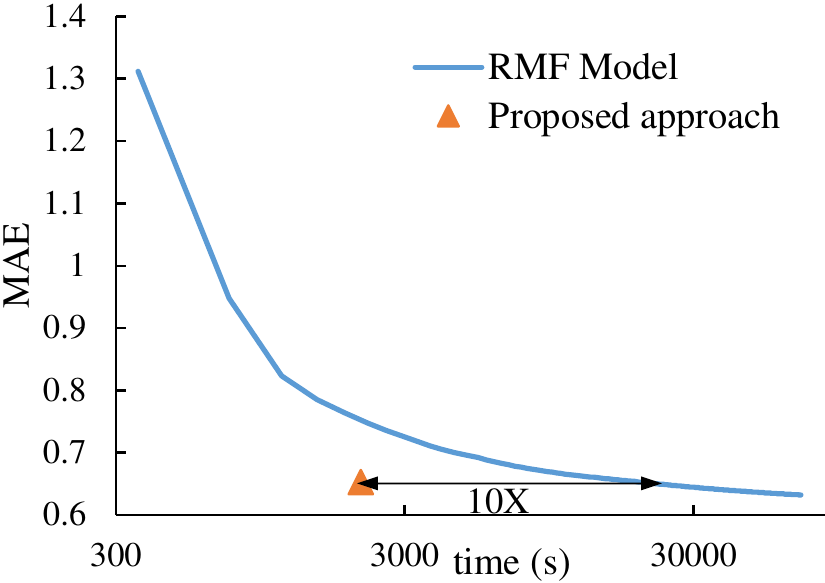}
\caption{The MAE of the RMF model based collaborative filtering approach on MovieLens-20M vs. its computational time. The comparison with the proposed approach shows that the latter is 10X faster for achieving same accuracy. }
\end{figure}

It should be pointed out that, the proposed approach also enables choosing a smaller $k$ near the turning point after which the prediction error on validation set declines slowly. For the case shown in Fig. 2, this smaller $k$ may be 200. This could trade off a little bit accuracy for shorter runtime.

\section{Conclusions}

In this work, a model-based collaborative filtering approach with automatically determined latent factors is proposed. It includes a fast adaptive PCA framework for sparse data and a novel termination mechanism for incremental generation of latent factors. The approach is parameter free, and achieves better accuracy and effectiveness than other model-based CF approaches and the fast SVT algorithm. It also adapts to large sparse data in real recommendation problems. For reproducibility, the codes and test data in this work will be shared on GitHub (https://github.com/xindubawukong222/K-Adaptive-CF).

Although the model-based CF approach is not as accurate as the deep-learning-based CF approach, the former is much more computationally
efficient than the latter. This work also reveals that the difference of prediction error between them is marginal (mostly less than 5\%). So, in some scenarios (without GPUs or demanding environment-friendly)
the proposed model-based CF approach is useful and advantageous.


\end{document}